\documentclass[12pt,dvipsnames]{article}

\usepackage{libertine}
\usepackage{amsmath, amsthm, amssymb}
\usepackage{amssymb}
\usepackage{bbm}

\usepackage{hyperref}
\hypersetup{
    colorlinks=true,
    linkcolor=teal,
    filecolor=teal,      
    urlcolor=teal,
    citecolor=teal,
}

\usepackage{cleveref}
\usepackage{url}

\usepackage{graphicx}
\DeclareGraphicsExtensions{.pdf,.png,.jpg}

\usepackage[table]{xcolor}

\usepackage{epigraph}
\usepackage{pdflscape}
\usepackage{afterpage}
\usepackage[margin=4cm]{geometry}

\usepackage{pdfpages}

\usepackage{caption}
\usepackage[T1]{fontenc}
\usepackage[all]{xy}
\usepackage[inline]{enumitem}

\usepackage{tikz}
\usepackage{tikz-cd}
\usepackage{pgfplots}
\pgfplotsset{compat=newest}

\usepackage{siunitx}
\usepackage{subfig}
\makeatletter
\let\@fnsymbol\@arabic
\makeatother

\newtheorem{theorem}{Theorem}
\newtheorem{corollary}{Corollary}
\newtheorem{remark}{Remark}

\newtheorem{definition}{Definition}
\newtheorem{proposition}{Proposition}

\newcommand{\V}{{\mathcal{V}}}
\renewcommand{\P}{{\mathcal{P}}}
\renewcommand{\L}{{\mathcal{L}}}
\newcommand{\E}{{\mathcal{E}}}

\newcommand{\R}{{\mathbb R}}
\newcommand{\id}{{\textnormal{id}}}
\newcommand{\Lspace}{{L^2([t_0, T], n)}}

\newcommand{\range}[2]{{\{{#1}, \dots,{#2}\}}}
\newcommand{\anon}{{\,\mbox{-}\,}}
\crefname{diagram}{diag.}{diags.}
\Crefname{diagram}{Diagram}{Diagrams}
\creflabelformat{diagram}{#2(#1)#3}

\tikzstyle{vertex} = [fill,shape=circle,node distance=55pt]
\tikzstyle{edge} = [fill,opacity=.5,fill opacity=.5,line cap=round,
line join=round, line width=45pt]
\tikzstyle{elabel} =  [fill,shape=circle,node distance=24pt]

\pgfdeclarelayer{background}
\pgfsetlayers{background,main}
\title{
    Parametric machines: a fresh approach to architecture search
    }
\author{
    Pietro Vertechi
    \footnotemark[1]
    \and Mattia G. Bergomi
    \thanks{Correspondence \url{pietro.vertechi@protonmail.com}, \url{mattiagbergomi@gmail.com}}
}
\date{}

\begin{document}
\captionsetup[subfigure]{position=top, labelfont=bf,textfont=normalfont,singlelinecheck=off,justification=raggedright}
\maketitle

\begin{abstract}
    Using tools from topology and functional analysis, we provide a framework where artificial neural networks, and their architectures, can be formally described.
    We define the notion of {\em machine} in a general topological context and show how simple machines can be combined into more complex ones. We explore finite- and infinite-depth machines, which generalize neural networks and neural ordinary differential equations.
    Borrowing ideas from functional analysis and kernel methods, we build complete, normed, infinite-dimensional spaces of machines, and we discuss how to find optimal architectures and parameters---within those spaces---to solve a given computational problem.
    In our numerical experiments, these kernel-inspired networks can outperform classical neural networks when the training dataset is small.
\end{abstract}

\section{Introduction}

\paragraph{Background.} In recent years, the deep learning framework  achieved and surpassed state-of-the-art performance in many machine learning tasks, using a variety of architectures. Notably, in the field of computer-vision, Convolutional Neural Networks showcase impressive performance~\cite{krizhevskyImageNetClassificationDeep2017}.
However, a paradoxical problem affects the performance and robustness of deep neural networks. Deeper networks should in principle perform at least as well as shallower ones, finding in the limit of infinite layers a solution where the extra layers approximate the identity function. However, \cite{heDeepResidualLearning2016} reports that deeper architectures can cause a degradation of performance not explained by overfitting. Choosing a deep architecture is therefore a difficult task, where one needs to rely on heuristics, or brute trial and error. Current approaches to automated architecture search~\cite{elskenNeuralArchitectureSearch} rely on large or augmented training datasets and manually engineered building-blocks. Moreover, they often lack principled regularization methods and guarantees of optimality. A first step in the study of these building-blocks from a topological-geometrical point of view has been done in~\cite{bergomiTopologicalGeometricalTheory2019}, showing the relevance of the contraction property in managing parametric spaces of group equivariant non-expansive operators (GENEOs).

\paragraph{Aim.} Our ambition is to define architectures with little or no human intervention. We interpret a neural architecture as a continuous family of endomorphisms on a function space. From this perspective, it is possible to swiftly parameterize complex architectures (e.g., multi-scale convolutional networks, shortcut connections), and thereafter \textit{sculpt} them while backpropagating by eliminating uninformative, noisy, and redundant connections.
This procedure allows for extremely simplified flows for designing neural architectures, requiring the user only to specify data-type, loss function, and type and number of activation functions to be considered. With this information, we define an \textit{as generous as possible} neural architecture constrained by the equivariance implied by the data-type (e.g. convolution for images). We then sculpt this architecture during training.

\paragraph{Contributions.}
We propose a theoretical framework where neural networks can be formally described as a special case of a more general construction: {\em parametric machines}.
Modularity---a fundamental property of standard neural architectures---is intrinsic to this construction: it is possible to create complex machines as a sum of simpler ones.
Our notion unifies seemingly disparate architectures, ranging from hand-designed combinations of layers, graphically represented here via a hypergraph, to networks defined via differential equations~\cite{chenNeuralOrdinaryDifferential2018}.
The key intuition is that a neural network can be considered as an endomorphism $f$ on a space of global functions (defined on all neurons on all layers). If such a network is feedforward, then $\id - f$ is invertible, and its inverse can be computed via a forward pass. The two broad classes of architectures that we describe here are the analogous of the classical results that $\id - f$ is invertible if $f$ is a linear nilpotent map (finite depth) or a contraction (infinite depth).
Infinite-depth machines generalize neural ordinary differential equations, by adding a choice of architecture. Unlike the finite-depth case, whose structure can be represented by a hypergraph, this architecture is defined in terms of continuous functions and, therefore, can be parameterized and optimized during training.
When the training dataset is small, we rely on kernel methods to guarantee optimality.
Finite- and infinite-depth {\em kernel machines} exhibit {\em all} shortcut connections, thus avoiding pathologies due to the architecture depth. Such dense connectivity does not cause a quadratic increase in the number of parameters in the case of small datasets. In addition to the theoretical framework, we test our main algorithms, namely {\em hypergraph neural architecture sculpting}, and {\em discrete} and {\em continuous kernel machines}, in three applications, proving their effectiveness, with a focus on small datasets (i.e., less than 100 training samples). Each algorithm has been wrapped as a PyTorch~\cite{paszkeAutomaticDifferentiationPyTorch2017} module, and can be used both as standalone or {\em layer} of a classical neural network architecture.

\paragraph{Structure.}
\Cref{sec:smallmachines} discusses the necessary preliminaries. Building on those, we introduce the notion of {\em machine} and its {\em stable state}. These generalize the connection between global nonlinear operators on function spaces and the forward pass of a layered neural network or neural Ordinary Differential Equation~(ODE), see~\cref{sec:hypergraphmachines,sec:VolterraMachines} respectively.
In~\cref{sec:kernelmachines}, taking advantage of the framework developed in \cref{sec:smallmachines,sec:discretecontinuousneuralnetworks}, we define a novel architecture based on operator-valued kernels and filtrations of Hilbert spaces.
The proposed constructions are tested on different tasks and compared with state-of-the-art methods.

\section{Machines}
\label{sec:smallmachines}

We lay our fundamental definitions in the context of topological vector spaces and (potentially nonlinear) functions between them. Most spaces we consider will have this dual nature (topological and linear). We denote with $\times$ the product of topological space, and with $\oplus$ the product of topological vector spaces.

\subsection{Stable state}
\label{sec:stablestate}

We start by describing how, in the classical deep learning framework, different layers are combined to form a network.
Intuitively, function composition seems the natural operation to do so. A sequence of layers
\begin{equation*}
    X_0 \xrightarrow{l_1} X_1 \xrightarrow{l_2} \dots X_{n-1} \xrightarrow{l_n} X_n
\end{equation*}
is composed into a map $X_0 \rightarrow X_n$. However, this intuition breaks down in the case of shortcut connections or more complex, non-sequential architectures.

From a mathematical perspective, a natural alternative is to consider a {\em global} space $X = \bigoplus_{i=0}^n X_i$, and the global endofunction
\begin{equation*}
    f = \sum_{i=1}^n l_i \colon X \rightarrow X.
\end{equation*}
What remains to be understood is the relationship between the function $f$ and the layer composition $l_n\circ l_{n-1}\circ \dots\circ l_2\circ l_1$. To clarify this relationship, we assume that the output of the network is the entire space $X$, and not only the output of the last layer, $X_n$.
Let the input function be the inclusion $g\colon X_0 \rightarrow X$. The network transforms $g$ into a map $h\colon X_0 \rightarrow X$, induced by $l_i\circ \dots \circ l_1\colon X_0 \rightarrow X_i$, for $i \in \range{0}{n}$. From a practical perspective, $h$ computes the activation values of all the layers and stores not only the final result, but also all the activations of the intermediate layers.

The key observation, on which our framework is based, is that $f$ and $g$ alone are sufficient to determine $h$. Indeed, $h$ is the only map $X_0\rightarrow X$ that respects the following property:
\begin{equation}
    \label{eq:network}
    h = g + f h.
\end{equation}
\Cref{eq:network} holds also in the presence of shortcut connections, or more complex architectures such as UNet~\cite{liHDenseUNetHybridDensely2018} (see~\cref{sec:hypergraphmachines}).
The existence of a unique solution to~\cref{eq:network} for any choice of input function $g$ will be the defining property of a {\em machine}, our generalization of a feedforward deep neural network.

\begin{definition}
    \label{def:machine}
    Let $M$ be a topological vector space. A continuous endofunction $f \colon M \rightarrow M$ is a \emph{machine} if, for all continuous map $g\colon X \rightarrow M$, there exists a unique continuous map $h\colon X \rightarrow M$ such that:
    \begin{equation*}
        h = g + f h.
    \end{equation*}
    We call $h$ the \emph{stable state} of $f$ with initial condition $g$, and denote by $S_f$ the stable state of $f$ with initial condition $\id_{M}$.
\end{definition}
The following result will be crucial to compute stable states in the remainder of this work.
\begin{theorem}
    \label{thm:oneminusf}
    $f \colon M \rightarrow M$ is a machine if and only if $\,\id-f$ is an isomorphism. Whenever that is the case, the stable state with initial condition $g$ is given by $(\id-f)^{-1}\circ g$. In particular, the stable state of $f$ is $S_f = (\id-f)^{-1}$.
\end{theorem}
\begin{proof}
    Let us assume that $f$ is a machine. $S_f = \id + f S_f$, so $(\id - f) S_f = \id$, hence $\id - f$ is a split epimorphism. Let $h, h'$ be such that $(\id - f)h = (\id - f)h'$. Then both $h$ and $h'$ are stable states of $f$ with initial condition $(\id - f)h$, hence they must be equal, so $\id - f$ is monic. A monic split epimorphism is necessarily an isomorphism.
    Conversely, let us assume that $\id - f$ is an isomorphism. Then $h = g + f h$ if and only if $h = (\id-f)^{-1} g$.
\end{proof}

\paragraph{Parametric machines.}
Let $P$ be a topological {\em parameter space}. A {\em parametric machine} is simply a continuous family of machines $f_p(m) \colon P \times M \rightarrow M$ such that, given a continuous family of input functions $g_p(m)$, the family of stable states $h_p$ is also jointly continuous in $p$ and $m$. We call $f_p$ a {\em parametric machine}, with {\em parameter space} $P$.

\subsection{Convergence and depth}
\label{sec:convergence_and_depth}

All nilpotent linear endomorphisms of a Banach space are machines. Continuous endofunctions with norm strictly smaller than $1$ (i.e., not necessarily linear contractions) are also machines.
In both cases, the stable state can be found by considering the following sequence:
\begin{equation}
    \label{eq:machinesequence}
    h_0 = \id
    \qquad\text{ and }\qquad
    h_{n+1} = \id + fh_n.
\end{equation}
Even though for different reasons, both in the nilpotent, linear case and in the contraction case, $\lVert h_{m}-h_n\rVert$ converges to $0$ for sufficiently large $m,n$.
If $f$ is nilpotent and linear, then $h_{n+1}-h_n = f(h_{n}-h_{n-1})$, so it will go to $0$ in a finite number of steps. If instead $f$ has norm $\lambda < 1$, then
\begin{align*}
    \lVert h_{n+1}-h_n\rVert = \lVert fh_{n}-fh_{n-1}\rVert
    \le \lambda \lVert h_{n}-h_{n-1}\rVert.
\end{align*}
Therefore, consecutive distances are uniformly bounded by $c\lambda^n$ for some $c$, hence, for $m \ge n$, $\lVert h_{m}-h_n\rVert \le \frac{c\lambda^n}{1-\lambda}$, thus ensuring convergence.

\begin{definition}
    \label{def:depthmachine}
    Let $f, \{h_i\}_{i\in \mathbb N}$ be as in~\cref{eq:machinesequence}.
    The {\em depth} of $f$ is the smallest integer $n$ (if it exists) such that
    \begin{equation*}
        h_{n+1} = h_n,
    \end{equation*}
    and $\infty$ otherwise.
\end{definition}

\subsection{Modularity and computability}

Under suitable \emph{independence} conditions, more complex machines can be created as a sum of simpler ones.
\begin{definition}
    \label{def:independence}
    Let $M$ be a topological vector space. Let $f, f' \colon M \rightarrow M$ be continuous endofunctions. We say that $f$ \emph{does not depend} on $f'$ if, for any topological space $X$, for any pair of continuous maps $b, b'\colon X \rightarrow M$, and for all $\lambda \in \R$, the following holds:
    \begin{equation}
        \label{eq:independence}
        f (b + \lambda f' b') = f b.
    \end{equation}
    Otherwise, we say that $f$ depends on $f'$.
\end{definition}
\begin{remark}
    Independence of $f$ from $f'$ is stronger than asking $f f' = 0$, because in general it is not true that $f (a + a') = f a + f a'$.
\end{remark}
\Cref{def:independence} is quite useful to compute stable states. For example, if $f$ does not depend on itself, then automatically $f$ is a machine, and $S_f = \id + f$; we call such machines {\em square-zero}.
Under suitable assumptions, machines can be juxtaposed to recover the notion of \textit{deep neural networks}.
\begin{theorem}
    \label{thm:sumofsmallmachines}
    Let $f, f'$ be machines such that $f$ does not depend on $f'$. Then $f+f'$ is also a machine, and $S_{f+f'} = S_{f'} S_{f}$. If furthermore $f'$ does not depend on $f$, then $S_{f+f'} = S_f + S_{f'} - \id$.
\end{theorem}
\begin{proof}
    By~\cref{thm:oneminusf,eq:independence}, $f+f'$ is a machine:
    \begin{equation}
        \label{eq:compositionindependent}
        (\id - f)(\id - f') = (\id - f - f'),
    \end{equation}
    so $(\id - f - f')$ is an isomorphism (composition of isomorphisms). \Cref{eq:compositionindependent} also determines the stable state:
    \begin{equation*}
        S_{f+f'} = (\id - f - f')^{-1} = (\id - f')^{-1}(\id - f)^{-1} = S_{f'}S_f.
    \end{equation*}
    Moreover, if $f'$ does not depend on $f$, then
    \begin{align*}
        f(S_f + S_{f'} - \id) = f(S_f + f'S_{f'}) = f S_f,\\
        f'(S_f + S_{f'} - \id) = f'(fS_f + S_{f'}) = f' S_f'.
    \end{align*}
    Hence,
    \begin{equation*}
        S_f + S_{f'} - \id = \id + (f+f')(S_f + S_{f'} - \id).
    \end{equation*}
\end{proof}

\Cref{thm:sumofsmallmachines} allows us to build a broad class of networks from basic components.
Given a set of machines $\{f_1, \dots, f_n\}$, we can define its \emph{dependency graph} as follows: the set of vertices is $\{1, \dots, n\}$, and there is a directed edge from $i$ to $j$ (for $i \neq j$) if and only if $f_j$ depends on $f_i$.
If the dependency graph is acyclic, then $f_1+\dots+f_n$ is a machine, and there is an efficient procedure to compute its stable state. We will need some basic graph-theoretical notions to describe it.

\paragraph{Layering of acyclic directed graphs.} Given a finite directed graph $(V, E)$, a \emph{layering}~\cite{sugiyamaMethodsVisualUnderstanding1981} on $(V, E)$ of height~$k$ is a partition $\{V_1, \dots, V_k\}$ of its vertices such that, whenever we have an edge from $v_i\in V_i$ to $v_j\in V_j$, then necessarily $i<j$. A directed graph $(V, E)$ can only admit layerings if it is acyclic.
In that case, the height of a layering must be at least the length of the longest path in $(V, E)$ increased by one.
This lower bound is tight. Indeed, given a vertex $v\in V$, we can define its \emph{depth} $d(v)$ to be the length of the longest path terminating in $v$. A layering of minimal height can be defined as follows:
\begin{equation*}
    V_i = d^{-1}(i+1).
\end{equation*}

\begin{corollary}
    \label{cor:acyclicgraph}
    Let us consider a set of machines $\{f_1, \dots, f_n\}$, and let $(V, E)$ be its dependency graph. Let us assume that $(V, E)$ is acyclic, with layering $\{V_1, \dots, V_k\}$.
    If we denote $f = f_1+\dots+f_n$, then
    \begin{equation}
        \label{eq:computemachinegraph}
        S_f
        = \left(\id + \sum_{v \in V_k} \left(S_{f_v}-\id\right)\right) \dots \left(\id + \sum_{v \in V_1} \left(S_{f_v}-\id\right)\right).
    \end{equation}
\end{corollary}
\begin{proof}
    We can define a new set of machines $\{l_1, \dots, l_k\}$, where
    \begin{equation*}
        l_i = \sum_{v \in V_i} f_v.
    \end{equation*}
    For $i < j$, $l_i$ does not depend on $l_j$, so \cref{eq:computemachinegraph} follows trivially from~\cref{thm:sumofsmallmachines}:
    \begin{equation*}
        S_f
        = S_{l_k}\dots S_{l_1}
        = \left(\id + \sum_{v \in V_k} \left(S_{f_v}-\id\right)\right) \dots \left(\id + \sum_{v \in V_1} \left(S_{f_v}-\id\right)\right).
    \end{equation*}
\end{proof}

\Cref{cor:acyclicgraph} establishes a clear link between sums of independent machines and compositions of layers in classical feedforward neural networks. Even though, in general, we are not limited to sequential architectures (see~\cref{fig:hypergraphs}), the layering procedure determines the order in which machines should be concatenated.

\section{Finite and infinite depth}
\label{sec:discretecontinuousneuralnetworks}

Neural networks can be seen as a sum of independent square-zero machines, one per layer.
We first use our machine-based framework to design finite-depth architectures using directed hypergraphs. This allows for shortcut connections~\cite{bishopNeuralNetworksPattern1995,ripleyPatternRecognitionNeural1996}, as in, for instance, residual learning networks~\cite{heDeepResidualLearning2016}, as well as more complex connectivities, such as UNet~\cite{liHDenseUNetHybridDensely2018}.

Analogously, ODEs correspond to a sum of independent contracting machines, obtained by splitting the time interval into small sub-intervals.
This is a standard strategy to obtain existence and uniqueness results for ODEs, which are a consequence of the Caccioppoli-Banach principle~\cite[Chapt.~XVI]{kantorovichFunctionalAnalysis1982}---contractions in a complete metric space admit a unique fixed point.
As described in~\cref{sec:convergence_and_depth}, unlike square-zero machines, which have depth $1$, contracting machines can in general have infinite depth.
We describe {\em Volterra machines}, a generalization of neural ODEs~\cite{chenNeuralOrdinaryDifferential2018} in our framework, as an example of an infinite-depth machine.

\subsection{Hypergraph machines}
\label{sec:hypergraphmachines}

\begin{figure}
     \begin{tikzpicture}[scale=0.8, every node/.style={scale=0.8}]
       \node[vertex,label=above:\(v_1\)] (v1) {};
       \node[vertex,below right of=v1,label=above:\(v_3\)] (v3) {};
       \node[vertex,below left of=v3,label=above:\(v_2\)] (v2) {};
       \node[vertex,right of=v3,label=above:\(v_4\)] (v4) {};
       \node[vertex,right of=v4,label=above:\(v_5\)] (v5) {};
       \node[vertex,above right of=v4,label=above:\(v_6\)] (v6) {};
       \node[vertex,below right of=v4,label=above:\(v_7\)] (v7) {};
       \node[vertex,right of=v5,label=above:\(v_8\)] (v8) {};

       \begin{pgfonlayer}{background}
           \draw[edge,opacity=.3,color=Dandelion] (v1) -- (v2) -- (v3) -- (v1);
           \draw[edge,opacity=.3,color=red, line width=38pt] (v3) -- (v4);
           \draw[edge,opacity=.3,color=LimeGreen] (v4) -- (v6) -- (v5) -- (v7) -- (v4);
           \draw[edge,opacity=.3,color=blue, line width=38pt] (v5) -- (v8);
           \draw[edge,opacity=.3,color=cyan, line width=38pt] (v1) -- (v6);
       \end{pgfonlayer}

       \node[elabel,opacity=.3,color=Dandelion,label=right:\(E_1: \{v_1\mbox{,}v_2\}\rightarrow \{v_3\}\)]  (e1) at (-7.5,0.4) {};
       \node[elabel,opacity=.3,below of=e1,color=red,label=right:\(E_2: \{v_3\}\rightarrow \{v_4\}\)]  (e2) {};
       \node[elabel,opacity=.3,below of=e2,color=cyan,label=right:\(E_3: \{v_1\}\rightarrow \{v_6\}\)]  (e3) {};
       \node[elabel,opacity=.3,below of=e3,color=LimeGreen,label=right:\(E_4: \{v_4\}\rightarrow\{v_6\mbox{,}v_5\mbox{,}v_7\}\)]  (e4) {};
       \node[elabel,opacity=.3,below of=e4,color=blue,label=right:\(E_5: \{v_5\}\rightarrow \{v_8\}\)]  (e6) {};
       \end{tikzpicture}

     \caption{\textbf{Hypergraph representation of a neural network.} Given layers $\{l_1, \dots, l_5\}$, the representation corresponds to the neural network mapping $(x_1,\; x_2,\; x_3,\; x_4,\; \dots, x_8)$ to $(x_1,\; x_2,\; l_1(x_1, x_2) + x_3,\; l_2(l_1(x_1, x_2) + x_3) + x_4,\; \dots,\; l_5(l_4(l_2(l_1(x_1, x_2) + x_3) + x_4) + x_5) + x_8 )$.}
     \label{fig:pruning}
\end{figure}
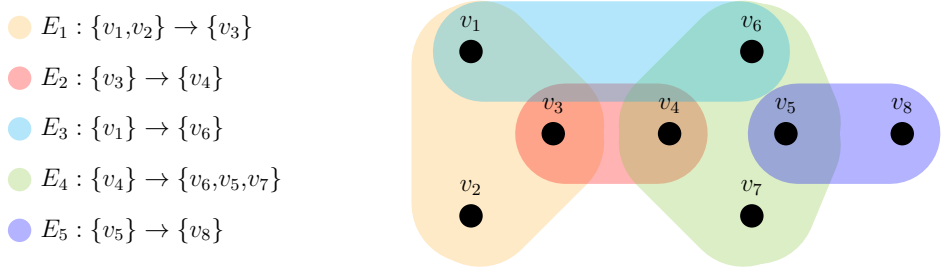

We will need some basic notions concerning directed hypergraphs from~\cite{galloDirectedHypergraphsApplications1993}.

\begin{definition}\label{def:hypergraph}\cite[Sect. 2]{galloDirectedHypergraphsApplications1993}
    Let $\P: \mathbf{Set} \to \mathbf{Set}$ denote the power set functor. A \textit{directed hypergraph} is a pair of finite sets $(\V, \E)$ of vertices and hyperedges, with $\E\subseteq \P\V\times\P\V$, that is to say each hyperedge $E$ can have several source vertices (or none) and several target vertices (or none). We denote the subset of source vertices and target vertices $s(E)$ and $t(E)$ respectively. In the remainder of this work, directed hypergraphs will simply be called hypergraphs.
\end{definition}

Even though~\cite{galloDirectedHypergraphsApplications1993} requires hyperedges to have disjoint source and target, we drop this condition. The notion of acyclic hypergraph is identical as hyperedges with overlapping source and target are cycles of length $1$.

\begin{definition}\label{def:hypergraphpath}\cite[Sect. 3]{galloDirectedHypergraphsApplications1993}
    Given a hypergraph $(\V, \E)$, a \emph{path} $P_{ab}$ of length $q$ is a sequence $v_1=a,E_1,v_2,E_2,\dots,E_q,v_{q+1}=b$, where:
    \begin{equation*}
        a\in s(E_1),\quad
        v_j\in t(E_{j-1})\cap s(E_j),\; j \in \range{2}{q},
        \quad\text{ and }\quad
        b\in t(E_q).
    \end{equation*}
    $P_{ab}$ is a \emph{cycle} if $b\in s(E_1)$. A hypergraph is \emph{acyclic} if it has no cycles.
\end{definition}
\begin{definition}
    The \emph{line} graph of a directed hypergraph $H = (\V, \E)$ is a directed graph having as nodes the set $\E$ of hyperedges of $H$. $E_1$ is connected to $E_2$ if and only if $t(E_1) \cap s(E_2) \neq \emptyset$.
\end{definition}

Let $(\V, \E)$ be an acyclic hypergraph. A \emph{nonlinear hypergraph representation} is, for each vertex $v\in\V$, a topological vector space $M_v$, and, for each hyperedge $E \in \E$, a continuous map:
\begin{equation*}
  \begin{tikzcd}
    \bigoplus_{v \in s(E)} M_v \arrow{r}{p_E}
    & \bigoplus_{v \in t(E)} M_v.
  \end{tikzcd}
\end{equation*}
Let $M := \bigoplus_{v\in\V} M_v$. Then $p_E$ can be extended to a machine on $M$:
\begin{equation*}
\begin{tikzcd}
    M \arrow{r}{}
    & \bigoplus_{v \in s(E)} M_v \arrow{r}{p_E}
    & \bigoplus_{v \in t(E)} M_v \arrow{r}{}
    & M
\end{tikzcd}
\end{equation*}
The dependency graph for $\{p_E\}_{E\in\E}$ is a subgraph of the line graph of $(\V, \E)$, and is therefore also acyclic, hence the endomorphism $\sum_{E\in\E} p_E$ is a machine.

\begin{figure}
      \subfloat[][\textbf{Prunable hypergraph, initialization.}]{\includegraphics[width=.45\textwidth]{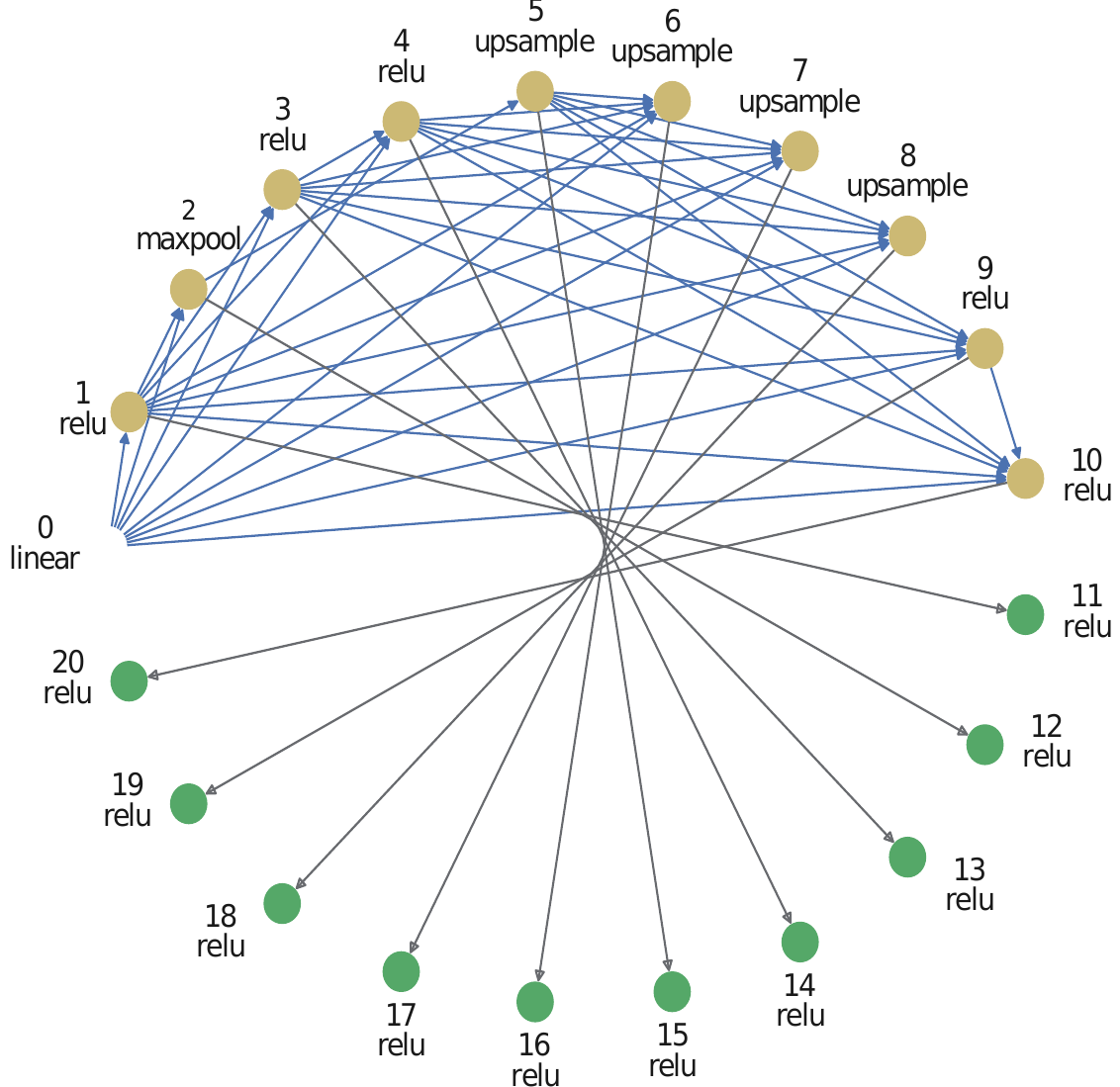}\label{fig:pruning_start}}
     \qquad
     \subfloat[][\textbf{Trained hypergraph, MNIST.}]{\includegraphics[width=.45\textwidth]{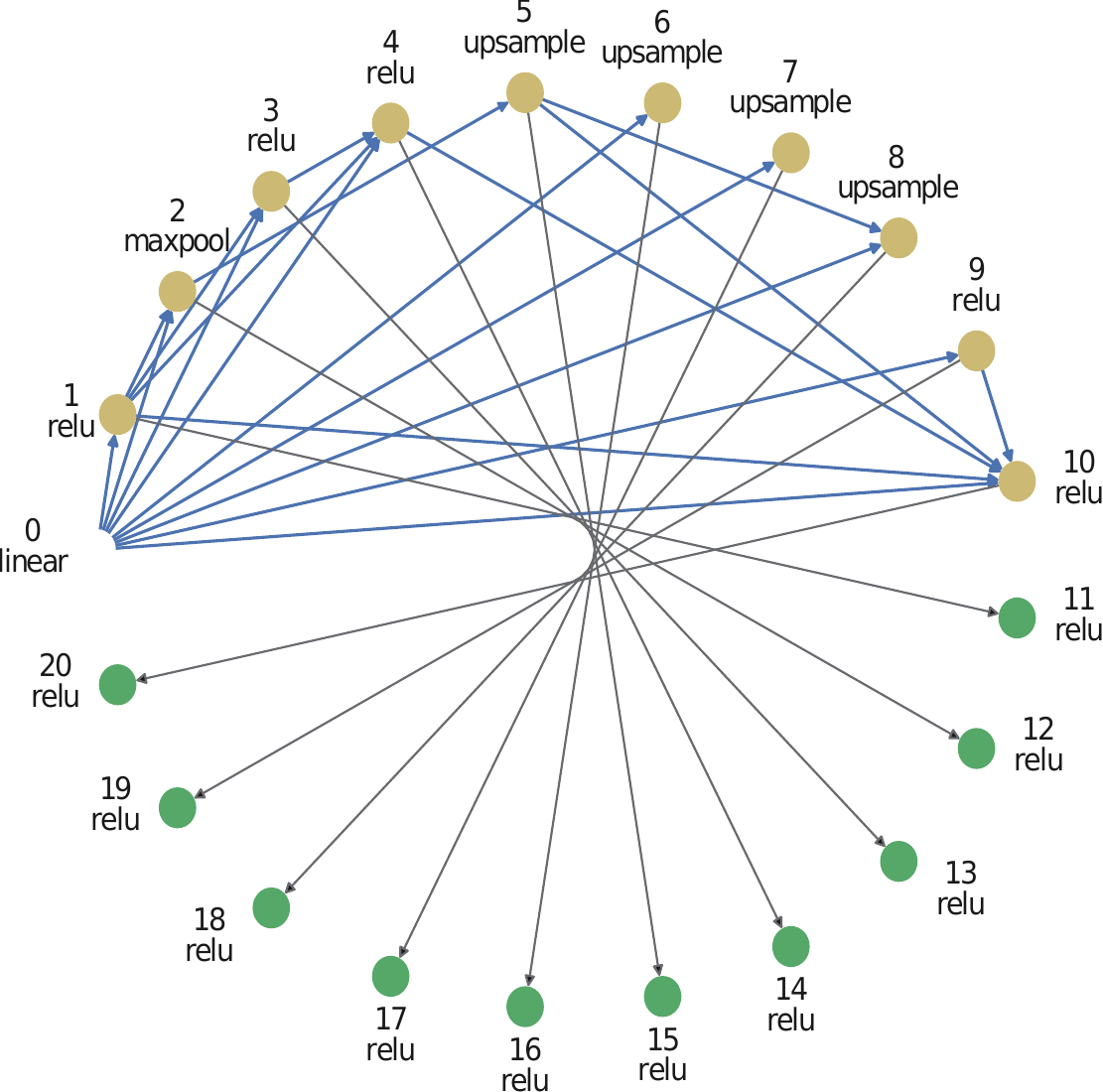}\label{fig:pruning_end}}

     \subfloat[][\textbf{Learned convolutional architecture.}]{\includegraphics[width=\textwidth]{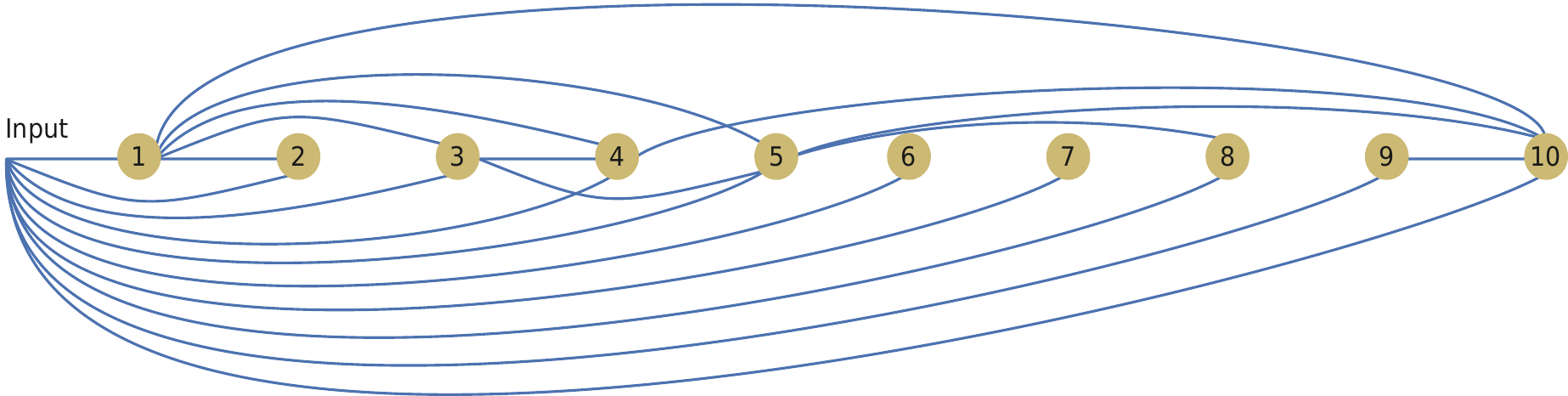}\label{fig:pruning_conv}}
     \caption{As a starting architecture we consider a directed acyclic graph whose nodes are activation functions (identity, ReLU, upsampling and max-pooling). Blue directed edges are convolutions and black directed edges are linear layers. In~\protect\subref{fig:pruning_start} we show the starting connectivity, which is maximal with respect to the blue edges which connect all admissible (i.e.~same dimensionality) activation nodes.
     During training, we prune those edges whose weights have sufficiently small Euclidean norm. \protect\subref{fig:pruning_end} Architecture after pruning during training on the MNIST dataset (with accuracy $\approx 98.6\%$). \protect\subref{fig:pruning_conv} The learned convolutional architecture.}
     \label{fig:hypergraphs}
\end{figure}

\paragraph{Hypergraph networks sculpting.}
Using the graph-theoretical ideas developed so far, we devised a first architecture sculpting algorithm that requires minimal user inputs and fine-tuning. We start with a finite number of nodes, each equipped with an activation function on a given space with a group of symmetries (i.e., translations for convolution, identity for fully-connected layers). Each node is connected to all preceding nodes with compatible dimensionality and has a unique fully-connected output. When reaching a node, the outputs of its incoming edges are summed. During training, we add to the loss function a cost proportional to the sum of the Euclidean norms of the weights associated with each edge. In \cref{fig:pruning_start} we show this construction for a translation-equivariant architecture used to classify the MNIST dataset~\cite{deng2012mnist}, where we start with $10$ nodes equipped with activation functions compatible with an image analysis task, connected by convolutional edges with a fixed number of channels. During training, we prune edges whose associated weights have Euclidean norm smaller than a fixed tolerance ($10^{-6}$), see~\cref{fig:pruning_end}. This small tolerance value has minimal impact on the accuracy of the model while reducing its computational cost. In~\cref{fig:pruning_conv}, we observe that the learned convolutional architecture has non-trivial connectivity. The achieved accuracy on the MNIST test set~($\approx 98.6\%$) is below state of the art. However, this particular algorithm does not require any manual fine-tuning, other than the choice of equivariance and number and dimension of nodes, which could be chosen automatically according to the computational power of the user's machine. A PyTorch implementation of this algorithm is available at \url{https://github.com/LimenResearch/hypergraph_machines}.

\subsection{Volterra machines}
\label{sec:VolterraMachines}

A natural generalization of neural ODEs in our framework is given by {\em Volterra machines}. The nonlinear {\em Volterra equation of the second kind} is, in its classical form:
\begin{equation}
    \label{eq:Volterra}
    u(t) = \psi(t) + \int_{t_0}^t \phi(t, s, u(s))ds,
    \text{ for all } t \in [t_0, T],
\end{equation}
where $t_0 < T \in \R$.
This equation generalizes ordinary differential equations. Whenever $\phi$ only depends on the last two arguments, i.e. $\phi(t, s, v) = \phi(s, v)$, and $\psi(t) = \psi(t_0)$ for all $t \in [t_0, T]$, then the solution $u$ of the Volterra equation (if it exists) also solves the initial value problem:
\begin{equation*}
    \frac{du(t)}{dt} = \phi(t, u(t))
    \text{ and } u(t_0) = \psi(t_0).
\end{equation*}

We consider the vector-valued case, where the codomain of $\phi, \psi$ (and consequently $u$) is the finite-dimensional Hilbert space $\R^n$, equipped with the standard scalar product.
Let $\Lspace$ be the Hilbert space of square-integrable functions from the interval $[t_0, T]$ to $\R^n$. We deviate slightly from the more standard set of assumptions (see~\cite{bakerPerspectiveNumericalTreatment2000}) to ensure existence and uniqueness of solutions, as we do not ask that $\psi$ is continuous:
\begin{enumerate}
    \item $\psi \in \Lspace$.
    \item $\phi(t, s, v)$ is continuous for $t_0 \le s \le t \le T$.
    \item $\phi(t, s, v)$ satisfies a uniform Lipschitz condition in $v$ for $t_0 \le s \le t \le T$. That is to say, there exists $\lambda \in \R$ such that, for all $v, \tilde v\in\R^n$,
    \begin{equation}
        \label{eq:Lipschitz}
        \lVert \phi(t, s, v) - \phi(t, s, \tilde{v})\rVert \le
        \lambda \lVert v-\tilde{v}\rVert.
    \end{equation}
\end{enumerate}
We show existence and uniqueness of solutions for square-integrable functions in the machine framework.
\begin{definition}
    \label{def:smallvolterramachine}
    Let $\phi(t, s, v)$ be a continuous function on $t_0 \le s \le t \le T$ and $v \in \R^n$, with values in $\R^n$. If $\phi(t, s, v)$ satisfies a uniform Lipschitz condition in $v$ for $t_0 \le s \le t \le T$, we say that $\phi$ is a {\em Volterra machine} on $\Lspace$.
\end{definition}
A Volterra machine $\phi$ is a machine on $\Lspace$.
Let
\[
f\colon \Lspace \rightarrow \Lspace
\]
be the nonlinear endofunction given by:
\begin{equation*}
    f(u) = t \mapsto \int_{t_0}^t \phi(t, s, u(s))ds.
\end{equation*}
Let $\lambda$ be such that~\cref{eq:Lipschitz} holds. Let us choose a positive integer $N$ such that
\begin{equation}
    \label{eq:largeN}
    N > \lambda^2(T-t_0)^2.
\end{equation}
For $i \in \range{0}{N}$, let $t_i = t_0 + \frac{i}{N}(T-t_0)$. For $i \in \range{1}{N}$, we can define
\begin{equation*}
    f_i(u) =
        t \mapsto
        \int_{t_0}^t \phi(t, s, u(s))\mathbbm{1}_{[t_{i-1}, t_i]}(s)ds.
\end{equation*}
Clearly $f = f_1+\dots + f_N$. Furthermore, for $i < j$, $f_i$ does not depend on $f_j$. We need to show that $f_i$ is a contraction. Then, given $u, \tilde{u} \in \Lspace$, we have:
\begin{align*}
    \lVert f_i(u) - f_i(\tilde{u})\rVert_2^2
    &= \int_{t_0}^T \left\Vert\int_{t_0}^t \mathbbm{1}_{[t_{i-1}, t_i]}(s)[\phi(t, s, u(s)) -
    \phi(t, s, \tilde{u}(s)] ds \right\Vert_2^2 dt\\
    &\le \int_{t_0}^T \frac{T-t_0}{N}\int_{t_0}^t
    \left\Vert\phi(t, s, u(s)) - \phi(t, s, \tilde{u}(s))\right\Vert_2^2 dsdt\\
    &\le \int_{t_0}^T \frac{T-t_0}{N} \int_{t_0}^t
    \lambda^2\left\Vert u(s) - \tilde{u}(s)\right\Vert_2^2 ds dt\\
    &\le \int_{t_0}^T \lambda^2\frac{T-t_0}{N}
    \left\Vert u - \tilde{u}\right\Vert_2^2 dt \\
    &= \lambda^2\frac{\left(T-t_0\right)^2}{N}
    \left\Vert u - \tilde{u}\right\Vert_2^2
\end{align*}
Therefore, by~\cref{eq:largeN}, $f_i$ is a contraction.
As $f$ is a sum of machines with an acyclic dependency graph, it is also a machine on $\Lspace$ by~\cref{cor:acyclicgraph}. In particular, given a sequence $\{\psi_n\}_{n\in\mathbb N} \to \psi_{\infty}$ of square-integrable functions that converges in norm $L^2$ to $\psi_\infty$, for all $n\in\mathbb N \cup \infty$ there is a unique $u_n$ such that
\begin{equation*}
    u_n(t) = \psi_n(t) + \int_{t_0}^t \phi(t, s, u_n(s))ds,
    \text{ for all } t \in [t_0, T],
\end{equation*}
and the sequence $\{u_n\}_{n\in\mathbb N}$ converges in norm $L^2$ to $u_\infty$.

\subsubsection{Efficient Volterra machines}
\label{sec:efficientvolterramachines}

Nonlinear Volterra integral equations are in general harder to solve than ordinary differential equations (see~\cite{bakerPerspectiveNumericalTreatment2000} for a review of possible methods).
This is particularly problematic here, as we wish to solve a Volterra equation in a time comparable with the forward pass of a neural ODE.
Luckily, some special cases of Volterra equations admit a simpler solution in terms of a system of ODEs~\cite{bowndsTheoryPerformanceSubroutine1982}.
Let $U, V, W$  be finite real vector spaces equipped with a bilinear map
$B\colon U \otimes V \rightarrow W$.
Let $\phi_1, \dots, \phi_m$ be $U$-valued functions, and $c_1, \dots, c_m$ $V$-valued functions. We can consider:
\begin{equation*}
    \phi(t, s, v) = \sum_{j=1}^m B(\phi_j(s, v), c_j(t)).
\end{equation*}
Analogously to a result presented in~\cite{bowndsTheoryPerformanceSubroutine1982}, we can solve the corresponding Volterra equation as a system of ODEs.
\begin{theorem}\cite[Thm.~3]{bowndsTheoryPerformanceSubroutine1982}
    \label{thm:fastvolterra}
    Let $\psi \in \Lspace$. Let
    \begin{equation*}
        \phi(t, s, v) = \sum_{j=1}^m B(\phi_j(s, v), c_j(t)).
    \end{equation*}
    Let $\range{z_1}{z_m}$ be the solution to the following system of ODEs:
    \begin{equation}
        \label{eq:odesystem}
        \frac{dz_j(t)}{dt} = \phi_j(t, u(t))\text{ for all } t \in [t_0, T],
        \text{ with } z_j(t_0) = 0,
    \end{equation}
    where
    \begin{equation*}
        u(t) = \psi(t) + \sum_{j=1}^m B(z_j(t), c_j(t)).
    \end{equation*}
    Then, $u, \phi, \psi$ respect \cref{eq:Volterra}.
\end{theorem}
\begin{proof}
    Integrating~\cref{eq:odesystem}, we obtain
    \begin{equation*}
        z_j(t) = \int_{t_0}^t \phi_j(s, u(s))ds.
    \end{equation*}
    Therefore:
    \begin{align*}
        u(t)
        &= \psi(t) + \sum_{j=1}^m B(z_j(t), c_j(t))\\
        &= \psi(t) + \sum_{j=1}^m \int_{t_0}^t B(\phi_j(s, u(s)), c_j(t))ds\\
        &= \psi(t) + \int_{t_0}^t\sum_{j=1}^m B(\phi_j(s, u(s)), c_j(t))ds\\
        &= \psi(t) + \int_{t_0}^t \phi(t, s, u(s))ds.
    \end{align*}
\end{proof}
This can be seen as a continuous analog of neural architecture search. Given a family of Neural ODEs $\range{\phi_1}{\phi_m}$, and functions $\range{c_1}{c_m}$, we can compute a loss function with respect to the Volterra machine $$\sum_{j=1}^m B(\phi_j(s, u(s)), c_j(t)).$$ From this perspective, the relative strengths of $c_j(t)$ can be interpreted as {\em routing}. We will give an application of Volterra machines in~\cref{sec:continuouskernelmachines}, in the context of kernel methods.

\section{Kernel machines}
\label{sec:kernelmachines}

We are interested in combining kernel methods~\cite{scholkopfLearningKernelsSupport2002} with the machine framework. In their simplest form, kernel methods associate to an input space $X$ a Hilbert space $H$ of real-valued functions defined on $X$. Here, however, we are interested in studying Hilbert spaces of endofunctions of $X$. To do so, we will need some notions from the theory of {\em operator-valued} kernel methods~\cite{alvarezKernelsVectorValuedFunctions2012,kadri2016operator,micchelli2005learning}.

\subsection{Operator-valued kernels}

Let $X$ be a space, and $Y$ a Hilbert space, with scalar product $\langle \anon, \anon \rangle$. We are interested in studying functions $X \rightarrow Y$. In the remainder, we will denote the set of functions from a space $X$ to another space $Y$ by $Y^X$.
Let $\L(Y)$ be the space of bounded linear endomorphisms of $Y$. It is a Banach space, with norm given by the operator norm.

\begin{definition}\cite[Def.~3]{kadri2016operator}
    \label{def:operatorkernel}
    Let $Y$ be a Hilbert space. A map $K\colon X \times X \rightarrow \L(Y)$ is an {\em operator-valued kernel} if the following conditions are satisfied.
    \begin{enumerate}
        \item For all $x_1, x_2\in X$ the operator $K(x_1, x_2)\colon Y \rightarrow Y$ is self-adjoint.
        \item For all $x_1, \dots, x_n\in X$, $c_1,\dots, c_n \in Y$, the matrix
        \begin{equation*}
            M_{i, j} = \langle c_i, K(x_i, x_j)c_j\rangle
        \end{equation*}
        is positive-semidefinite.
    \end{enumerate}
\end{definition}

\begin{remark}
    A scalar kernel on $X$ can always be seen as an operator-valued kernel $K\colon X \times X \rightarrow \L(Y)$, where for all $x_1, x_2\in X$, $K(x_1, x_2)$ is a multiple of the identity.
\end{remark}

An operator-valued kernel $K\colon X \times X \rightarrow \L(Y)$ will induce a feature map $X \rightarrow H \subseteq Y^X$, where $H$ is the {\em Reproducing Kernel Hilbert Space} (RKHS~\cite{aronszajnTheoryReproducingKernels1950}) associated to $K$.
In particular, $H$ is a space of $Y$-valued functions on $X$. Every function in $H$ can be written as a sum:
\begin{equation*}
    f(x) = \sum_{j=1}^{\infty} K(x, x_j) c_j,
\end{equation*}
where, for every $j$, $x_j \in X$ and $c_j \in Y$.
Even though the above sum has infinite elements, this is never a problem in practice.
Given a function $f\in H$ and a finite dataset $\range{x_1}{x_m}$, one can always find $\range{c_1}{c_m}$ such that, for all $x \in \range{x_1}{x_m}$,
\begin{equation*}
    f(x) = \sum_{j=1}^{m} K(x, x_j) c_j.
\end{equation*}
In general machine learning problems, the function $\sum_{j=1}^{m} K(\anon, x_j) c_j$ is preferable to $f$ as, even though they are indistinguishable on the training dataset, we have
\begin{equation*}
    \left\Vert\sum_{j=1}^{m} K(\anon, x_j) c_j\right\Vert_H \le \left\Vert f\right\Vert_H,
\end{equation*}
and hence $\sum_{j=1}^{m} K(\anon, x_j) c_j$ tends to be smoother and better behaved.

As $H$ is a space of functions from $X$ to $Y$, we have a canonical map $H \times X \rightarrow Y$, given by function evaluation. In what follows, we will focus on the case $X = Y$.

\begin{definition}
    \label{def:kernelmachine}
    Let $X$ be a Hilbert space. Let $K\colon X\times X \rightarrow \L(X)$ be an operator-valued kernel, with RKHS $H$. $K$ is a {\em kernel machine} if the canonical map
    \begin{equation*}
        H \times X \rightarrow X
    \end{equation*}
    is a parametric machine.
\end{definition}

\Cref{def:kernelmachine} implies that for all $f \in H$, the function $f$ is a machine on $X$. Furthermore, one can use standard techniques from kernel methods to learn a function $f \in H$ whose associated stable state optimizes some relevant quantity.
In the case of kernel machines, an analog of the representer theorem~\cite{kimeldorf1971some} holds.
\begin{theorem}
  \label{thm:representer}
    Let us consider a finite set $S = \range{s_1}{s_m}$, a map $g\colon S \rightarrow X$, and a function $\Lambda\colon X^m\times \R \rightarrow \R$ strictly increasing in the last variable. Any solution to the optimization problem
    \begin{equation}
        \label{eq:kernelmachineobjective}
        \min_{f\in H} \Lambda(h(s_1), \dots, h(s_m), \lVert f \rVert_H),
    \end{equation}
    where $h$ is the stable state of $f$ with initial condition $g$, is of the form
    \begin{equation*}
        f(x) = \sum_{j=1}^{m} K(x, h(s_j)) c_j.
    \end{equation*}
\end{theorem}
\begin{proof}
    Let us consider one solution $f$. Let $h$ be its stable state with initial condition $g$, and let $x_j = h(s_j)$, for $j\in\range{1}{m}$. Let $\tilde f$ be the projection of $f$ on the subspace:
    \begin{equation*}
        \left\{K(\anon, x_1) c_1 + \dots +  K(\anon, x_m) c_m \,|\, c_1, \dots, c_m \in X\right\}.
    \end{equation*}
    We start by observing that, for each $j\in\{1, \dots, m\}$, $\tilde f(x_j) = f(x_j)$. $h$ is the stable state of $\tilde f$ with initial condition $g$, as, for every $j\in\{ 1, \dots, m\}$,
    \begin{equation*}
        h(s_j) = g(s_j) + f(h(s_j)) = g(s_j) + \tilde f(h(s_j)).
    \end{equation*}
    As a consequence, $\tilde f$ produces a value smaller or equal than $f$ in~\cref{eq:kernelmachineobjective}, with equality if and only if they have the same norm, that is to say
    \begin{equation*}
        f \in \left\{K(\anon, x_1) c_1 + \dots +  K(\anon, x_m) c_m \,|\, c_1, \dots, c_m \in X\right\}.
    \end{equation*}
\end{proof}

In the context of kernel machines and for very small datasets, \cref{thm:representer} can be applied directly, guaranteeing optimality. In practice, for medium or large datasets, standard downsampling techniques, such as Nystr\"om sampling~\cite{drineasNystromMethodApproximating2005}, could be applied to replace $S = \range{s_1}{s_m}$ with a smaller subset of anchor points $\tilde S = \range{\tilde s_1}{\tilde s_{\tilde m}}$, with $\tilde m < m$.

In the following \cref{sec:discretekernelmachines,sec:continuouskernelmachines}, we will give two classes of examples of kernel machines, based on {\em discrete} and {\em continuous} filtrations of a Hilbert space.

\subsection{Finite depth kernel machines}
\label{sec:discretekernelmachines}

\begin{figure}
 \centering

 \subfloat[][\textbf{Fitting a 2D polynomial.}]{\includegraphics[width=.43\textwidth]{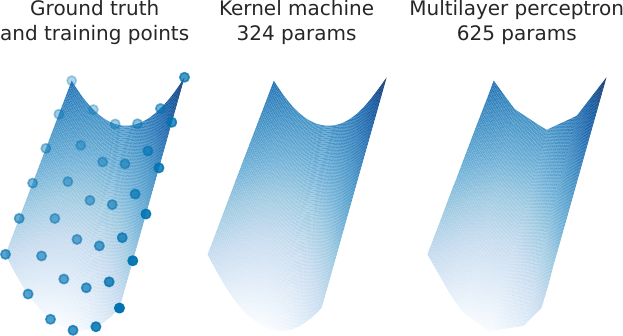}\label{fig:surfaces}}\qquad
 \subfloat[][\textbf{Performance over training.}]{\includegraphics[width=.43\textwidth]{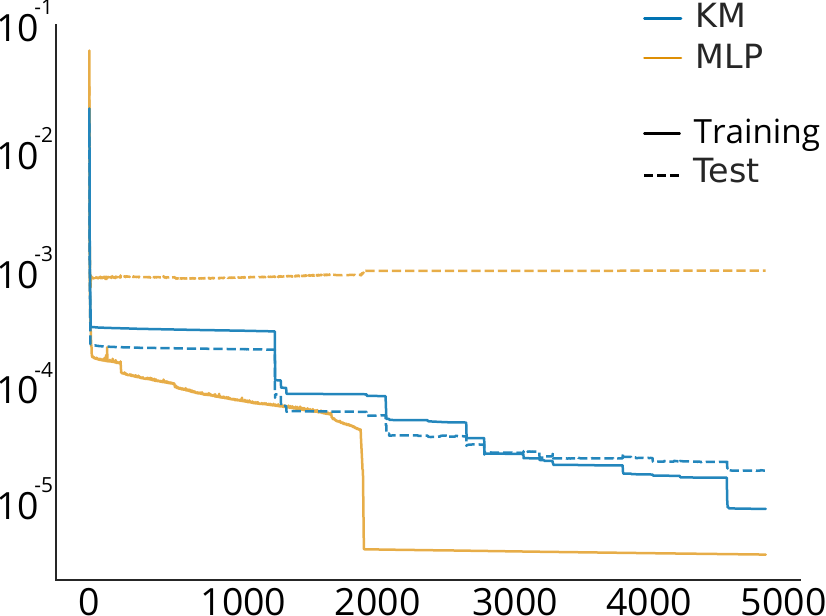}\label{fig:loss}}

 \subfloat[][\textbf{Interpolating from noisy
 data.}]{\includegraphics[width=.9\textwidth]{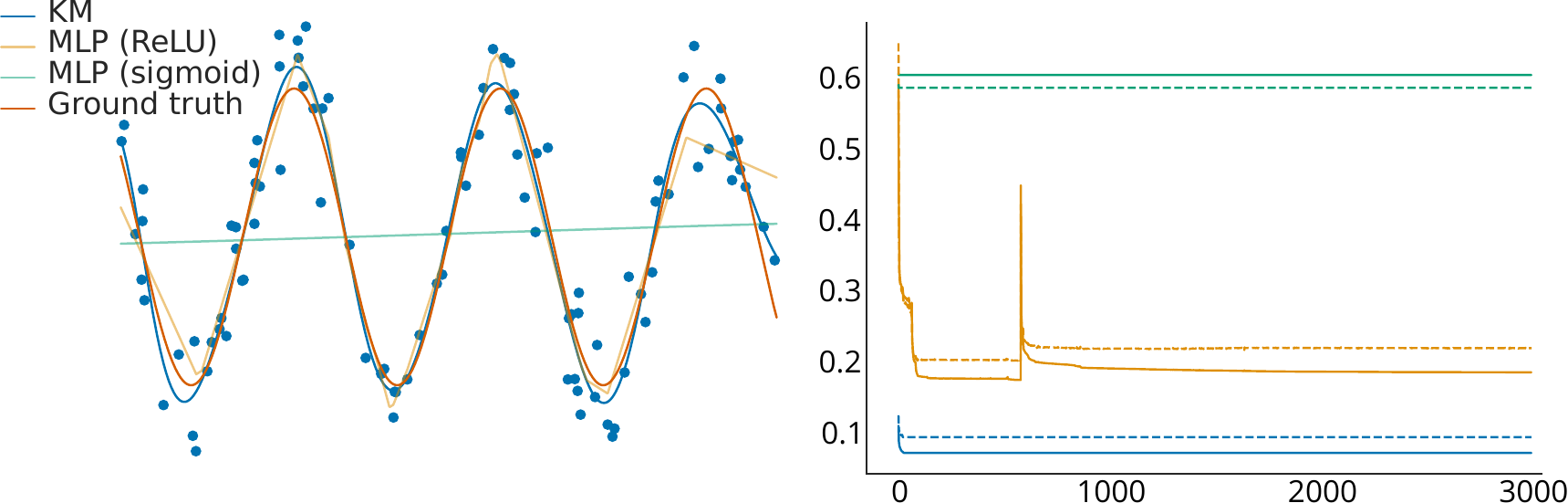}\label{fig:sine}}

 \subfloat[][\textbf{Regularization and loss.}]{\includegraphics[width=.9\textwidth]{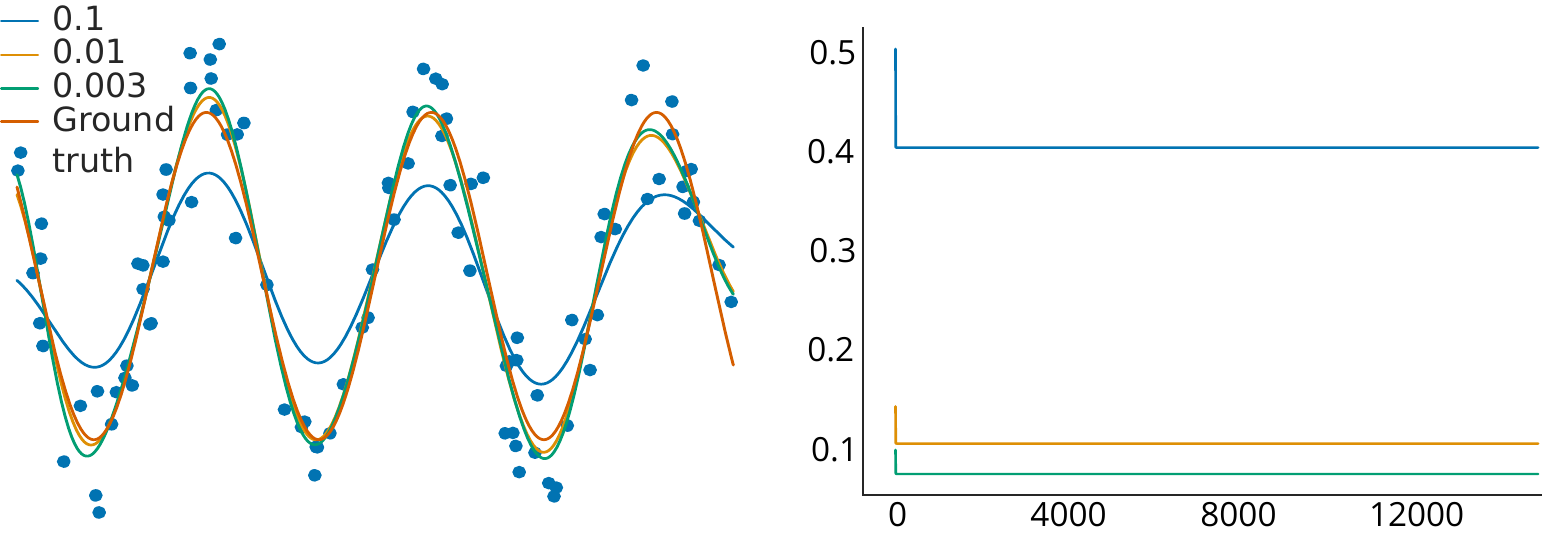}\label{fig:cost}}

 \caption{\textbf{Performance of finite-depth kernel machines.} We trained a kernel network and a multilayer perceptron with the same number of trainable parameters to fit a polynomial in two variables on a $6\times 6$ grid of points. While both achieve good performance, the kernel machine shows better decoding~\protect\subref{fig:surfaces} and a smaller loss function on the validation set (dashed line)~\protect\subref{fig:loss}. In~\protect\subref{fig:sine} we test robustness to noise of the kernel machine (514 parameters) in a noisy interpolation problem, comparing it with a 2 layers perceptron (609 parameters) with ReLU and sigmoid nonlinearities, respectively. In~\protect\subref{fig:cost} we show how different regularization coefficients affect the performance of the kernel machine.}
 \label{fig:discretekernelmachine}
\end{figure}

We associate a kernel machine to an arbitrary Hilbert space equipped with a finite filtration of closed subspaces.
\begin{definition}
    \label{def:discretecompatible}
    Let $X$ be a Hilbert space, equipped with a finite filtration of closed subspaces
    \begin{equation*}
        0 = X_0 \subseteq X_1 \subseteq X_2 \dots \subseteq X_n \subseteq X_{n+1} = X.
    \end{equation*}
    Let us consider a family of operator-valued kernels
    \begin{equation*}
        K_i\colon X_{i} \times X_{i} \rightarrow \L(X_{i+1} \cap X_i^\perp) \text{ for } i\in\range{0}{n}.
    \end{equation*}
    The {\em sum kernel machine} is given by
    \begin{equation*}
        K = \sum_{i=0}^n K_i.
    \end{equation*}
\end{definition}
The decomposition $K = \sum_{i=0}^n K_i$ corresponds to a decomposition of the RKHS $H \simeq \bigoplus_{i=0}^n H_i$, where, for every $i$, $H_i$ is the RKHS of $K_i$. In particular, given an endofunction $f\in H$, we have a unique decomposition $f = f_0 + \dots + f_n$, where $f_i \in H_i$ for all $i\in\{0,\dots, n\}$.
\begin{proposition}
    \label{prop:discretefiltrationkernel}
    Let $K$ be a sum kernel machine, and let $H$ be the corresponding RKHS.
    The application map
    \begin{align*}
        \varrho \colon H \times X &\rightarrow X\\
        (f, x) &\mapsto f(x)
    \end{align*}
    is a parametric machine. As a consequence, each endofunction $f \in H$ is a machine.
\end{proposition}
\begin{proof}
    Let us write:
    \begin{equation*}
        \varrho = \varrho_0+\dots + \varrho_n,
    \end{equation*}
    where, for $i\in\range{0}{n}$, $\varrho_i$ is the application map corresponding to $K_i$.
    It is straightforward to show that, for $i_1 \le i_2$, $\varrho_{i_1}$ does not depend on $\varrho_{i_2}$. In particular, each $\varrho_i$ is square-zero, and thus a machine. Moreover, the dependency graph of $\{\varrho_0,\dots, \varrho_n\}$ is acyclic, as the source of each edge always has a smaller index than the target. It follows from~\cref{cor:acyclicgraph} that $\varrho_0+\dots + \varrho_n$ is a machine, whose stable state can be computed via \cref{eq:computemachinegraph}.
\end{proof}

In classical terms, kernel machines in $H$ correspond to a network with $n+1$ layers and {\em all} shortcut connections. While in classical deep neural networks this would cause an explosion in the number of parameters, which would grow quadratically with the number of layers, in the case of small datasets and kernel machines this is not the case. A general kernel machine, on a training set with $m$ datapoints $\range{s_1}{s_m}$, can be expressed as:
\begin{equation}
    \label{eq:fullkernelmachine}
    X\ni x \mapsto \sum_{i=0}^n\sum_{j=1}^m K_i(x, h(s_j)) c_j,
\end{equation}
where $h$ is the stable state of the kernel machine (see~\cref{thm:representer}).
Each $c_j$ is a vector of $\textnormal{dim}(X)$ free parameters.
Therefore the number of parameters, $m \cdot\textnormal{dim}(X)$, grows linearly, rather than quadratically, with the number of layers.

\paragraph{Finite-depth kernel machines on small datasets.}
Small datasets are the natural testbed for finite-depth kernel machines given the architecture described by~\cref{eq:fullkernelmachine} and optimality guarantees obtained in~\cref{thm:representer}.
We implemented this architecture as a PyTorch module (implementation available at  \url{https://github.com/LimenResearch/kernel_machines}) and as a Julia~\cite{bezansonJuliaFreshApproach2017} package (\url{https://github.com/LimenResearch/KernelMachines.jl}), which relies on Zygote.jl~\cite{innesDonUnrollAdjoint2019} for automatic differentiation and on Optim.jl~\cite{mogensen2018optim} for optimization methods.
We chose to work with radial basis function kernels of the form
\[
K(u,v) = \exp(-\Vert u - v\Vert^2).
\]
We first test the architecture on a surface-fitting task, with ground truth $p(x,y) = (2x - 1)^2 + 2 y + xy - 3$. The training set consists of 36 points obtained by evaluating $p$ on a uniform $6\times 6$ grid in $[0,1]^2$. Test points are randomly chosen in the same domain (see~\cref{fig:surfaces}). We report the performance of the kernel machine (324 parameters) in~\cref{fig:loss} and compare it with a two-layers perceptron (625 parameters). Although both architectures are regularized, we can observe how the perceptron's performance is affected by overfitting, while the kernel machine reaches similar loss values on the training and test set. We then test the same kernel machine on the interpolation of noisy data, see~\cref{fig:sine}. Again, we compare its performance against 2-layer perceptrons with ReLU and sigmoid activation functions, respectively. We train on 100 random points obtained by sampling from a noisy sine. The kernel machine reaches the best performance on both the training and the validation set. Finally, on the same task, we test in~\cref{fig:cost} the robustness of the kernel machine to variation of the regularization cost.

\subsection{Infinite depth kernel machines}
\label{sec:continuouskernelmachines}

To translate the discrete filtration kernel described in~\cref{sec:discretekernelmachines} to the continuous case, we replace the discrete filtration with a continuous one. Let $X$ be a Hilbert space, $t_0 < T \in \R$, and
\begin{equation}
    \label{eq:continuousfiltration}
    0 = X_{t_0} \subseteq \dots
    \subseteq X_t \subseteq
    \dots \subseteq X_T = X
\end{equation}
a filtration of closed subspaces of $X$. We need a technical assumption to proceed in the continuous case.
\begin{definition}
    \label{def:pointwisecontinuity}
    Let $X$ be a Hilbert space. Let $\{X_t\}_{t \in [t_0, T]}$ be a filtration on $X$, and let $\pi_t$ denote the orthogonal projection on $X_t$, for $t \in [t_0, T]$. We say that $\{X_t\}_{t \in [t_0, T]}$ is {\em continuous} if, for all $x \in X$, the function
    \begin{align*}
        [t_0, T] &\rightarrow X \\
        t &\mapsto \pi_t(x)
    \end{align*}
    is continuous with respect to the norm on $X$.
\end{definition}

\begin{theorem}
    \label{thm:continuouskernelmachine}
    Let $X$ be a Hilbert space, with a continuous filtration $\{X_t\}_{t \in [t_0, T]}$, and corresponding orthogonal projections $\{\pi_t\}_{t \in [t_0, T]}$. Let $K\colon X \times X \rightarrow \L(X)$ be an operator-valued kernel, and $H$ be its RKHS. Finally, let
    \begin{equation*}
        \varrho \colon H \times X \rightarrow X
    \end{equation*}
    be the application map. Let us assume that
    \begin{itemize}
        \item  the distance induced by $K$ is bounded by a multiple of the norm-induced distance on $X$,
        \item for all $t \in [t_0, T]$, for all $x_1, x_2 \in X$,
        \begin{equation*}
            K(x_1, x_2) \pi_t = \pi_t K(x_1, x_2) = \pi_t K(\pi_t x_1, \pi_t x_2).
        \end{equation*}
    \end{itemize}
    Then $\varrho$ is a parametric machine, that is to say $K$ is a kernel machine.
\end{theorem}
\begin{proof}
    Let $\lambda > 0$ be such that, for all $x_1, x_2 \in X$,
    \begin{equation*}
        \lVert K(x_1, x_1)+K(x_2, x_2)-2K(x_1, x_2)\rVert_{\L(X)}
        \le \lambda^2\lVert x_1 - x_2\rVert_X^2.
    \end{equation*}
    Let
    \begin{equation*}
        H \ni \xi = \sum_{j=1}^\infty K(\anon, x_j) c_j.
    \end{equation*}
    Let $m$ be such that
    \begin{equation*}
        \left\Vert \xi - \hat{\xi} \right\Vert \le \frac{1}{4\lambda},
        \text{ where }
        \hat{\xi} := \sum_{j=1}^m K(\anon, x_j) c_j.
    \end{equation*}
    As the filtration $\{X_t\}_{t \in [t_0, T]}$ is continuous, for all $c\in X$, the map $t \mapsto \pi_t(c)$ is continuous and, therefore, uniformly continuous. In addition, $\pi_t$ commutes with $K$ by hypothesis, therefore we can choose $N$ such that, given $t_i := t_0 + \frac{i}{N}(T-t_0)$, for $i\in\range{1}{N}$,
    \begin{equation*}
        \left\Vert \left(\pi_{t_{i}}-\pi_{t_{i-1}}\right) \hat{\xi} \right\Vert
        = \left\Vert \sum_{j=1}^m K(\anon, x_j)
            \left(\pi_{t_{i}}-\pi_{t_{i-1}}\right)c_j \right\Vert
        \le \frac{1}{4\lambda}.
    \end{equation*}
    Let $B\left(\xi, \frac{1}{4\lambda}\right)$ be an open ball of radius $\frac{1}{4\lambda}$ around $\xi$. For each $\tilde{\xi} \in B\left(\xi, \frac{1}{4\lambda}\right)$, and for all $i\in\range{1}{N}$,
    \begin{equation*}
        \left\Vert \left(\pi_{t_{i}}-\pi_{t_{i-1}}\right) \tilde{\xi} \right\Vert \le \frac{3}{4\lambda}.
    \end{equation*}
    For all $i\in\range{1}{N}$, for all $\tilde{\xi} \in B\left(\xi, \frac{1}{4\lambda}\right)$, let $\tilde{\xi}_{i} = \left(\pi_{t_{i}}-\pi_{t_{i-1}}\right) \tilde{\xi}$. Then, for all $x_1, x_2, c \in X$,
    \begin{align*}
        \langle \tilde{\xi}_{i} (x_1-x_2), c\rangle_{X}^2
        &= \langle \tilde{\xi}_{i}, K(\anon, x_1)c - K(\anon, x_2)c\rangle_{H}^2 \\
        &\le \lVert \tilde{\xi}_{i} \rVert_{H}^2 \cdot
            \lVert K(\anon, x_1)c - K(\anon, x_2)c\rVert_{H}^2 \\
        &\le \left(\frac{3}{4\lambda}\right)^2 \cdot \langle c,
            \left(K(x_1, x_1)+K(x_2, x_2)-2K(x_1, x_2)\right)c\rangle_{X}\\
        &\le \left(\frac{3}{4\lambda}\right)^2 \cdot \lVert c \rVert_X^2
            \lVert K(x_1, x_1)+K(x_2, x_2)-2K(x_1, x_2)\rVert_{\L(X)}\\
        &\le \left(\frac{3}{4}\right)^2 \lVert c \rVert_X^2
            \lVert x_1 - x_2\rVert_X^2.
    \end{align*}
    By choosing $c = \tilde{\xi}_{i} (x_1-x_2)$, it follows that
    \begin{equation*}
        \lVert\tilde{\xi}_{i} (x_1-x_2)\rVert \le \frac{3}{4}\lVert x_1 - x_2\rVert,
    \end{equation*}
    hence $\tilde{\xi}_{i}$ is a contraction. For $i_1 < i_2$, $\tilde{\xi}_{i_1}$ does not depend on $\tilde{\xi}_{i_2}$, therefore
    \begin{equation*}
        \tilde{\xi} = \tilde{\xi}_{1}+\dots+\tilde{\xi}_{N}
    \end{equation*}
    is a machine, thanks to~\cref{cor:acyclicgraph}. As the same computing procedure can be applied to to all $\tilde{\xi}$ in a neighborhood of $\xi$, we have shown that $\varrho\colon H \times X \rightarrow X$ is a continuous parametric machine, hence $K$ is a kernel machine.
\end{proof}

\begin{figure}

 \subfloat[][\textbf{Performance on reduced MNIST varying the architecture}]{\includegraphics[width=.9\textwidth]{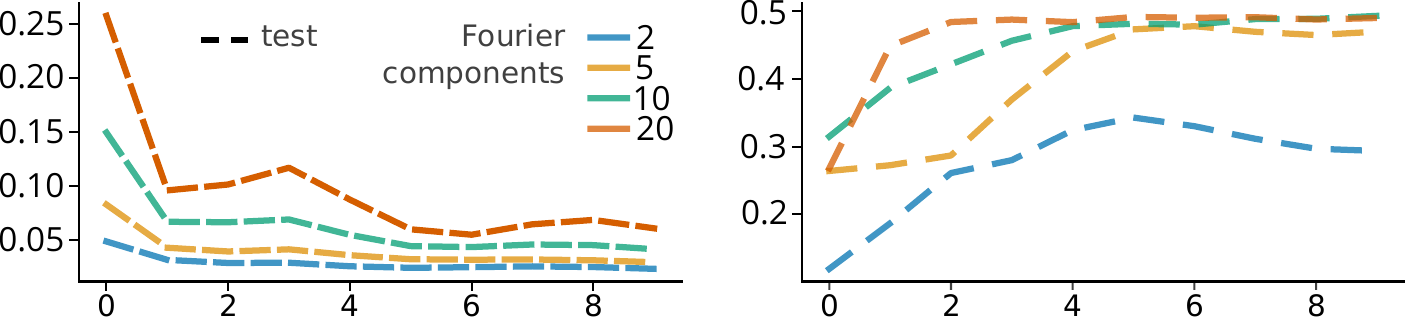}\label{fig:ckm_max_freq_loss}}

\subfloat[][\textbf{$c_j$ histograms}]{\includegraphics[width=.9\textwidth]{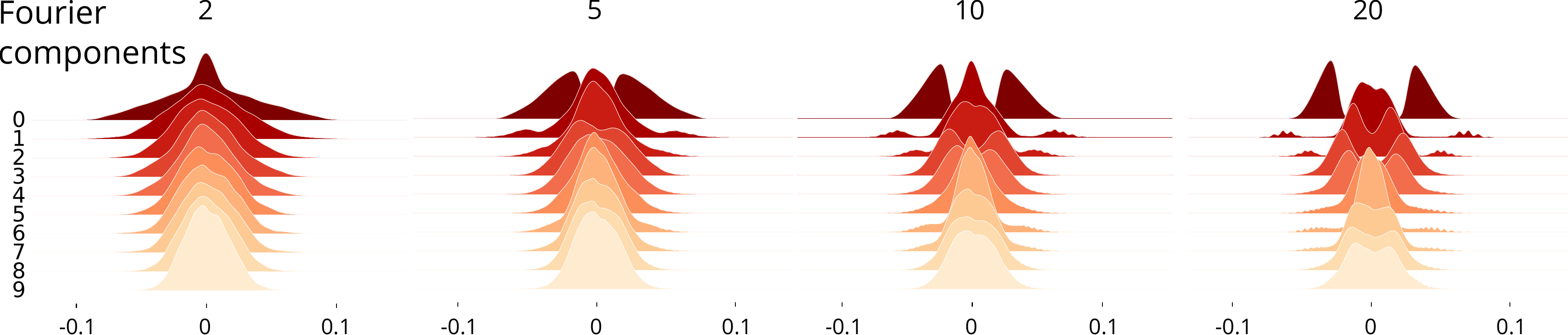}\label{fig:ckm_max_freq_hists}}

\subfloat[][\textbf{Performance on reduced MNIST, varying the cost}]{\includegraphics[width=.9\textwidth]{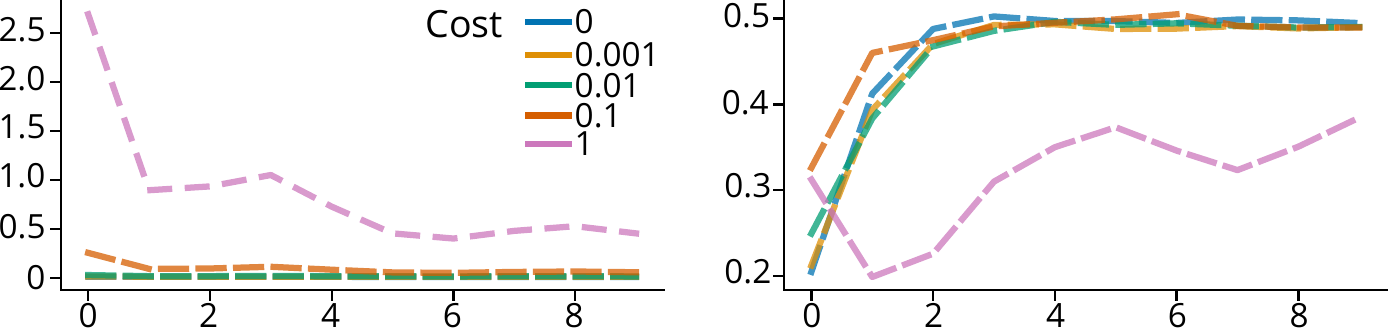}\label{fig:ckm_cost_loss}}

\subfloat[][\textbf{$c_j$ histograms}]{\includegraphics[width=.9\textwidth]{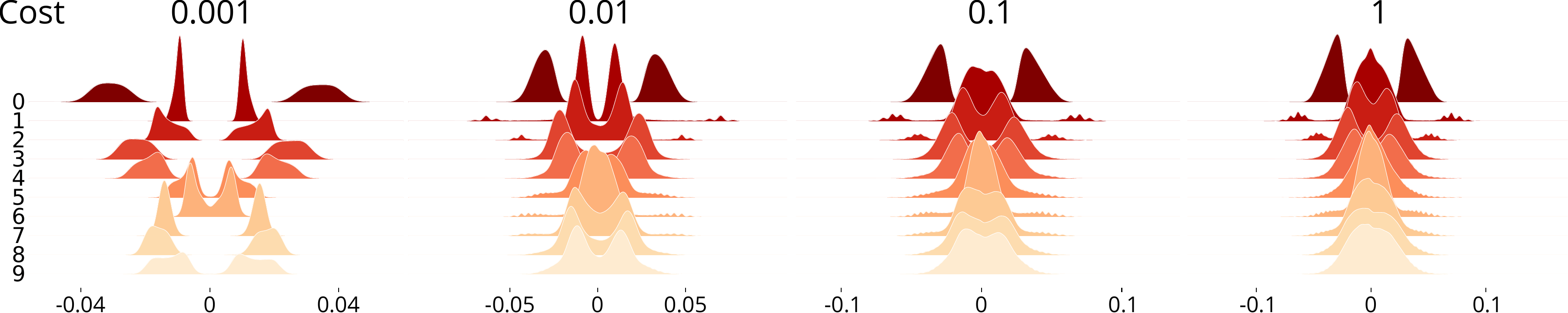}\label{fig:ckm_cost_hists}}

 \caption{\textbf{Performance of infinite-depth kernel machines.} We trained an infinite-depth kernel machine on one random sample per class of the MNIST dataset and tested it on the 10000 test samples. We approximated $c_j$ with a truncated Fourier series. In~\protect\cref{fig:ckm_max_freq_loss} we report the test loss and accuracy obtained varying the number of Fourier components used to approximate $c_j$ and fixing the regularization cost to $0.1$. According to our expectation an increase in number of Fourier components corresponds to an increase in performance (best $\approx 50\%$). \protect\Cref{fig:ckm_max_freq_hists} shows how $c_j$ evolves during training (epochs on the $y$-axis). Setting the number of Fourier components to 20, the change in performance caused by a variation of the regularization cost is reported in~\protect\cref{fig:ckm_cost_loss}, and in~\protect\cref{fig:ckm_cost_hists} the corresponding $c_j$ histograms.}
 \label{fig:continuouskernelmachine}
\end{figure}

\subsubsection{Computing continuous kernel machines efficiently}

Let us consider a particular case of filtration on a Hilbert space. Let $X = \Lspace$, and for all $t\in[t_0, T]$ let $X_t = L^2([t_0, t], n)$. Let $k\colon \R^n \times \R^n \rightarrow \L(\R^n)$ be a finite-dimensional continuous operator-valued kernel. We can consider the following operator-valued kernel
\begin{equation*}
    \left(K(x_1, x_2)x\right)(t) = \int_{t_0}^t k(x_1(s), x_2(s)) x(t) ds,
\end{equation*}
where $x_1,x_2,x\in X$ and $t\in [t_0, T]$.
Let $H$ be the RKHS corresponding to $K$. Let us further assume that the distance induced by $k$ on $\R^n$ is bounded by a multiple of the Euclidean distance. Then the distance induced by $K$ on $X$ is bounded by a multiple of the $L^2$ distance. By~\cref{thm:continuouskernelmachine}, $H \times X \rightarrow X$ is a parametric machine.

Let $f = \sum_{j=1}^m K(\anon, x_j) c_j$ be a machine in $H$. Let $\psi\in \Lspace$ be an initial condition. The stable state is given by the solution to the following Volterra equation:
\begin{equation}
    \label{eq:stablestatecontinuouskernel}
    u(t) = \psi(t) + \int_{t_0}^t \sum_{j=1}^m k(u(s), x_j(s))c_j(t)ds,
    \text{ for all } t \in [t_0, T],
\end{equation}
which can be computed efficiently using \cref{thm:fastvolterra}.

\paragraph{Infinite-depth kernel machines on small datasets.}
\Cref{eq:stablestatecontinuouskernel} gives an efficient way to implement and compute infinite-depth kernel machines. This construction satisfies the optimality result obtained in~\cref{thm:representer}. Thus, we implemented infinite-depth kernel machines as a PyTorch module and tested them on small datasets (implementation available at \url{https://github.com/LimenResearch/volterra_machines}). Intuitively, an infinite-depth kernel machine is a continuous architecture adding to state-of-the-art implementation, such as neural ODEs, all shortcut connections in time. Thus, we test infinite-depth kernel machines on a reduced version of the MNIST dataset~\cite{deng2012mnist} obtained considering one random sample per class (i.e. ten training images). In the function space defined by the machine, the architecture is chosen by selecting an incomplete basis on which the parameters of the machine are expressed. In our simulations, we consider radial basis kernel functions and an incomplete Fourier basis. In~\cref{fig:ckm_max_freq_loss}, we report the performance (loss and accuracy on the 10000 image MNIST test set) of the kernel machine when varying the architecture, i.e. varying the number of considered Fourier components. As expected, an increase in the number of such components causes an increase in performance. An histogram of the parameters $c_j$ (see~\cref{eq:stablestatecontinuouskernel}) is shown in~\cref{fig:ckm_max_freq_hists}. \Cref{fig:ckm_cost_loss,fig:ckm_cost_hists} show the change in performance of the kernel machine while varying the regularization cost.

\section{Conclusions}

We provide a solid topological and functional foundation for the study of deep neural networks. Borrowing ideas from functional analysis and graph theory, we define the abstract notion of {\em machine}, whose {\em stable state} generalizes the computation of a feedforward neural network. It is a unified concept that encompasses both manually designed neural network architectures, as well as their continuous counterpart such as Neural ODEs~\cite{chenNeuralOrdinaryDifferential2018}.

We take as starting point linear and nonlinear continuous maps between topological vector spaces. This alternation between linear and nonlinear components is one of the key ingredients of the success of deep neural networks, as it allows one to obtain complex functions as a composition of simpler ones. The notion of composition of layers in neural networks is unfortunately somewhat ill-defined, especially in the presence of shortcut connections and non-standard architectures. In the proposed {\em machine framework}, the composition is replaced by the sum. We describe independence conditions to ensure that the sum of machines is again a machine, in which case we can compute its stable state (forward pass) explicitly. This may seem counterintuitive, as the sum is a commutative operation, whereas the composition is not. However, in our framework, it is the {\em dependency graph} of a collection of machines that determines the order of composition.

Basic combinations of simple machines---{\em square-zero} and {\em contracting}---cover a lot of ground.
In particular, using finite sums of square-zero machines (discrete architectures), we recover classical neural networks, including architectures with shortcut connections.
In this setting, we provide a first simple application. Starting with a convolutional network with maximal connectivity, the architecture is automatically sculpted during training until a minimal architecture with robust performance is found. This algorithm is available as a PyTorch module.
Contracting, {\em infinite-depth} architectures generalize neural ODEs~\cite{chenNeuralOrdinaryDifferential2018}. More generally, we prove that, under some Lipschitz and continuity conditions, nonlinear integral Volterra equations of the second kind are machines. We provide an efficient procedure, with corresponding PyTorch implementation, to solve such equations in a special case.

Our approach meshes well with {\em deep kernel learning}~\cite{choKernelMethodsDeep2009,leeDeepNeuralNetworks2017,mairalEndtoEndKernelLearning2016,mairalConvolutionalKernelNetworks2014,nealPriorsInfiniteNetworks1996}, an attempt to combine modern advances in deep learning with classical kernel methods~\cite{hofmannKernelMethodsMachine2008}.
We believe this is particularly promising when working with small datasets, a scenario where deep neural networks have traditionally been less successful.
We introduce the notion of {\em kernel machine}, a Hilbert space whose points are machines.
There, given a specific loss function, we can search for machines that minimize it and that have a small norm. Even though the space is potentially infinite-dimensional, we prove an analog of the representer theorem, which determines a finite-dimensional subspace where optimal solutions can be found. This subspace can be quite large in practice. However, the norm can be used to regularize solutions.

We propose and implement in PyTorch two examples of kernel machines, with finite- and infinite-depth.
First, using kernels on finite filtrations of Hilbert spaces, we build finite-depth kernel machines. They correspond to neural networks with all shortcut connections.
In our simulations, with a comparable number of trainable parameters, kernel machines outperform multilayer perceptrons in toy problems with no more than $100$ training data points.
Second, using continuous filtrations on function spaces, we build infinite-depth kernel machines. While preserving the advantages of a kernel-based approach (optimality guarantees), infinite-depth kernel machines introduce the concept of shortcut connection in neural ODEs. Indeed, given a kernel, the value of the stable state (output) of the machine at time $t$ is obtained considering a restriction of the kernel to the interval $[t_0, t]$, i.e. all shortcuts up to time $t$.

The parameters to be optimized are functions in a Hilbert space.
As mentioned above, the infinite-dimensional function space represents a continuous architecture with {\em all} shortcut connections.
Different subspaces of this large function space correspond to different classical architectures. For instance, discrete architectures can be recovered via the subspace of piecewise constant functions on a given grid (grid-based approximation). As a consequence, a key ingredient of our method is the choice of a sculpting strategy to reduce the dimensionality of the function space. We explore two distinct, complementary approaches. First, we manually select an incomplete (finite) basis to make the problem tractable. Rather than limiting the {\em depth} of our architecture by approximating its parameters on a grid, we choose a low-frequency approximation, working with truncated Fourier series. Then, we further reduce the subspace of viable parameter values by penalizing points of the Hilbert space with a large norm.
This is, to the best of our knowledge, a novel approach to {\em Neural Architecture Search}~\cite{elskenNeuralArchitectureSearch}, where different architectures can be chosen (and compared) by selecting a scalar product and an incomplete basis of an infinite-dimensional function space.

\section*{Author contributions}

P.V. and M.G.B devised the project. P.V. and M.G.B developed the mathematical framework. P.V. and M.G.B. developed the software to implement the framework. P.V. wrote the original draft. M.G.B. reviewed and edited.

\bibliographystyle{abbrv}
\bibliography{SmallMachines}

\begin{thebibliography}{10}

\bibitem{alvarezKernelsVectorValuedFunctions2012}
M.~A. {\'A}lvarez, L.~Rosasco, and N.~D. Lawrence.
\newblock Kernels for {{Vector}}-{{Valued Functions}}: {{A Review}}.
\newblock {\em Foundations and Trends\textregistered{} in Machine Learning},
  4(3):195--266, June 2012.

\bibitem{aronszajnTheoryReproducingKernels1950}
N.~Aronszajn.
\newblock Theory of {{Reproducing Kernels}}.
\newblock {\em Transactions of the American Mathematical Society},
  68(3):337--404, 1950.

\bibitem{bakerPerspectiveNumericalTreatment2000}
C.~T.~H. Baker.
\newblock A perspective on the numerical treatment of {{Volterra}} equations.
\newblock {\em Journal of Computational and Applied Mathematics},
  125(1):217--249, Dec. 2000.

\bibitem{bergomiTopologicalGeometricalTheory2019}
M.~G. Bergomi, P.~Frosini, D.~Giorgi, and N.~Quercioli.
\newblock Towards a topological\textendash geometrical theory of group
  equivariant non-expansive operators for data analysis and machine learning.
\newblock {\em Nature Machine Intelligence}, pages 1--11, Sept. 2019.

\bibitem{bezansonJuliaFreshApproach2017}
J.~Bezanson, A.~Edelman, S.~Karpinski, and V.~B. Shah.
\newblock Julia: {{A Fresh Approach}} to {{Numerical Computing}}.
\newblock {\em SIAM Review}, 59(1):65--98, Jan. 2017.

\bibitem{bishopNeuralNetworksPattern1995}
C.~M. Bishop and P.~o. N. C. C.~M. Bishop.
\newblock {\em Neural {{Networks}} for {{Pattern Recognition}}}.
\newblock {Clarendon Press}, Nov. 1995.

\bibitem{bowndsTheoryPerformanceSubroutine1982}
J.~M. Bownds.
\newblock Theory and performance of a subroutine for solving {{Volterra
  Integral Equations}}.
\newblock {\em Computing}, 28(4):317--332, Dec. 1982.

\bibitem{chenNeuralOrdinaryDifferential2018}
T.~Q. Chen, Y.~Rubanova, J.~Bettencourt, and D.~K. Duvenaud.
\newblock Neural {{Ordinary Differential Equations}}.
\newblock In S.~Bengio, H.~Wallach, H.~Larochelle, K.~Grauman,
  N.~{Cesa-Bianchi}, and R.~Garnett, editors, {\em Advances in {{Neural
  Information Processing Systems}} 31}, pages 6571--6583. {Curran Associates,
  Inc.}, 2018.

\bibitem{choKernelMethodsDeep2009}
Y.~Cho and L.~K. Saul.
\newblock Kernel {{Methods}} for {{Deep Learning}}.
\newblock In Y.~Bengio, D.~Schuurmans, J.~D. Lafferty, C.~K.~I. Williams, and
  A.~Culotta, editors, {\em Advances in {{Neural Information Processing
  Systems}} 22}, pages 342--350. {Curran Associates, Inc.}, 2009.

\bibitem{deng2012mnist}
L.~Deng.
\newblock The mnist database of handwritten digit images for machine learning
  research [best of the web].
\newblock {\em IEEE Signal Processing Magazine}, 29(6):141--142, 2012.

\bibitem{drineasNystromMethodApproximating2005}
P.~Drineas and M.~W. Mahoney.
\newblock On the {{Nystr\"om Method}} for {{Approximating}} a {{Gram Matrix}}
  for {{Improved Kernel}}-{{Based Learning}}.
\newblock {\em Journal of Machine Learning Research}, 6(Dec):2153--2175, 2005.

\bibitem{elskenNeuralArchitectureSearch}
T.~Elsken, J.~H. Metzen, and F.~Hutter.
\newblock Neural architecture search: A survey.
\newblock {\em arXiv preprint arXiv:1808.05377}, 2018.

\bibitem{galloDirectedHypergraphsApplications1993}
G.~Gallo, G.~Longo, S.~Pallottino, and S.~Nguyen.
\newblock Directed hypergraphs and applications.
\newblock {\em Discrete Applied Mathematics}, 42(2):177--201, Apr. 1993.

\bibitem{heDeepResidualLearning2016}
K.~He, X.~Zhang, S.~Ren, and J.~Sun.
\newblock Deep {{Residual Learning}} for {{Image Recognition}}.
\newblock In {\em Proceedings of the {{IEEE Conference}} on {{Computer Vision}}
  and {{Pattern Recognition}}}, pages 770--778, 2016.

\bibitem{hofmannKernelMethodsMachine2008}
T.~Hofmann, B.~Sch{\"o}lkopf, and A.~J. Smola.
\newblock Kernel {{Methods}} in {{Machine Learning}}.
\newblock {\em The Annals of Statistics}, 36(3):1171--1220, 2008.

\bibitem{innesDonUnrollAdjoint2019}
M.~Innes.
\newblock Don't {{Unroll Adjoint}}: {{Differentiating SSA}}-{{Form Programs}}.
\newblock {\em arXiv:1810.07951 [cs]}, Mar. 2019.

\bibitem{kadri2016operator}
H.~Kadri, E.~Duflos, P.~Preux, S.~Canu, A.~Rakotomamonjy, and J.~Audiffren.
\newblock Operator-valued kernels for learning from functional response data.
\newblock {\em The Journal of Machine Learning Research}, 17(1):613--666, 2016.

\bibitem{kantorovichFunctionalAnalysis1982}
L.~V. Kantorovich and G.~P. Akilov.
\newblock {\em Functional Analysis}.
\newblock {Pergamon Press}, {Oxford ; New York}, 2d ed edition, 1982.

\bibitem{kimeldorf1971some}
G.~Kimeldorf and G.~Wahba.
\newblock Some results on tchebycheffian spline functions.
\newblock {\em Journal of mathematical analysis and applications},
  33(1):82--95, 1971.

\bibitem{krizhevskyImageNetClassificationDeep2017}
A.~Krizhevsky, I.~Sutskever, and G.~E. Hinton.
\newblock {{ImageNet}} classification with deep convolutional neural networks.
\newblock {\em Communications of the ACM}, 60(6):84--90, May 2017.

\bibitem{leeDeepNeuralNetworks2017}
J.~Lee, Y.~Bahri, R.~Novak, S.~S. Schoenholz, J.~Pennington, and
  J.~{Sohl-Dickstein}.
\newblock Deep {{Neural Networks}} as {{Gaussian Processes}}.
\newblock {\em arXiv:1711.00165 [cs, stat]}, Oct. 2017.

\bibitem{liHDenseUNetHybridDensely2018}
X.~Li, H.~Chen, X.~Qi, Q.~Dou, C.-W. Fu, and P.-A. Heng.
\newblock H-{{DenseUNet}}: {{Hybrid Densely Connected UNet}} for {{Liver}} and
  {{Tumor Segmentation From CT Volumes}}.
\newblock {\em IEEE Transactions on Medical Imaging}, 37(12):2663--2674, Dec.
  2018.

\bibitem{mairalEndtoEndKernelLearning2016}
J.~Mairal.
\newblock End-to-{{End Kernel Learning}} with {{Supervised Convolutional Kernel
  Networks}}.
\newblock In D.~D. Lee, M.~Sugiyama, U.~V. Luxburg, I.~Guyon, and R.~Garnett,
  editors, {\em Advances in {{Neural Information Processing Systems}} 29},
  pages 1399--1407. {Curran Associates, Inc.}, 2016.

\bibitem{mairalConvolutionalKernelNetworks2014}
J.~Mairal, P.~Koniusz, Z.~Harchaoui, and C.~Schmid.
\newblock Convolutional {{Kernel Networks}}.
\newblock In Z.~Ghahramani, M.~Welling, C.~Cortes, N.~D. Lawrence, and K.~Q.
  Weinberger, editors, {\em Advances in {{Neural Information Processing
  Systems}} 27}, pages 2627--2635. {Curran Associates, Inc.}, 2014.

\bibitem{micchelli2005learning}
C.~A. Micchelli and M.~Pontil.
\newblock On learning vector-valued functions.
\newblock {\em Neural computation}, 17(1):177--204, 2005.

\bibitem{mogensen2018optim}
P.~K. Mogensen and A.~N. Riseth.
\newblock Optim: A mathematical optimization package for {Julia}.
\newblock {\em Journal of Open Source Software}, 3(24):615, 2018.

\bibitem{nealPriorsInfiniteNetworks1996}
R.~M. Neal.
\newblock Priors for {{Infinite Networks}}.
\newblock In R.~M. Neal, editor, {\em Bayesian {{Learning}} for {{Neural
  Networks}}}, Lecture {{Notes}} in {{Statistics}}, pages 29--53. {Springer New
  York}, {New York, NY}, 1996.

\bibitem{paszkeAutomaticDifferentiationPyTorch2017}
A.~Paszke, S.~Gross, S.~Chintala, G.~Chanan, E.~Yang, Z.~DeVito, Z.~Lin,
  A.~Desmaison, L.~Antiga, and A.~Lerer.
\newblock Automatic differentiation in {{PyTorch}}.
\newblock Oct. 2017.

\bibitem{ripleyPatternRecognitionNeural1996}
B.~D. Ripley and N.~L. Hjort.
\newblock {\em Pattern {{Recognition}} and {{Neural Networks}}}.
\newblock {Cambridge University Press}, Jan. 1996.

\bibitem{scholkopfLearningKernelsSupport2002}
B.~Sch{\"o}lkopf, A.~J. Smola, M.~D. o. t. M. P. I. f. B. C. i. T. G. P.~B.
  Scholkopf, and F.~Bach.
\newblock {\em Learning with {{Kernels}}: {{Support Vector Machines}},
  {{Regularization}}, {{Optimization}}, and {{Beyond}}}.
\newblock {MIT Press}, 2002.

\bibitem{sugiyamaMethodsVisualUnderstanding1981}
K.~Sugiyama, S.~Tagawa, and M.~Toda.
\newblock Methods for {{Visual Understanding}} of {{Hierarchical System
  Structures}}.
\newblock {\em IEEE Transactions on Systems, Man, and Cybernetics},
  11(2):109--125, Feb. 1981.

\end{thebibliography}

\end{document}